\documentclass[conference]{IEEEtran}

\usepackage{times}
\usepackage[numbers]{natbib}
\usepackage{amsmath,amssymb,amsthm}
\usepackage{graphicx}
\usepackage{booktabs}
\usepackage{multirow}
\usepackage{xcolor}
\usepackage{xspace}
\usepackage{subcaption}
\usepackage{fvextra}
\usepackage[bookmarks=true,hidelinks]{hyperref}
\usepackage[capitalise,nameinlink]{cleveref}
\usepackage{cuted}

\graphicspath{{fig/}}

\definecolor{promptframe}{gray}{0.85}
\DefineVerbatimEnvironment{promptbox}{Verbatim}{%
    fontsize=\footnotesize,%
    frame=single,%
    framesep=4pt,%
    rulecolor=\color{promptframe},%
    breaklines=true,%
    breakautoindent=false,%
    breaksymbolleft={},%
    breaksymbolright={},%
}

\newtheorem{theorem}{Theorem}

\newtheorem{proposition}[theorem]{Proposition}

\renewcommand{\thetable}{\arabic{table}}

\newcommand{\ourmethod}{\textsc{Texedo}}
\newcommand{\basegen}{FSQ-GPT}
\newcommand{\Rtext}{R_{\mathrm{text}}}
\newcommand{\Rdyn}{R_{\mathrm{dyn}}}

\definecolor{base-gray}{RGB}{120,120,120}
\definecolor{rdyn-blue}{RGB}{13,121,202}
\definecolor{rtext-coral}{RGB}{255,119,94}
\definecolor{rdyn}{HTML}{829AE3}
\definecolor{rtext}{HTML}{A3C078}
\definecolor{oracle-color}{HTML}{977AE4}
\definecolor{linrew-aqua}{RGB}{52,172,139}
\definecolor{ours-orange}{HTML}{F37021}
\definecolor{gt-dark}{RGB}{60,60,60}
\definecolor{random-gray}{RGB}{180,180,180}

\newcommand{\base}{\textcolor{base-gray}{\textbf{Base}}\xspace}
\newcommand{\randomsel}{\textcolor{random-gray}{\textbf{Random}}\xspace}
\newcommand{\rdynonly}{\textcolor{rdyn}{\textbf{$\Rdyn$-only}}\xspace}
\newcommand{\rtextonly}{\textcolor{rtext}{\textbf{$\Rtext$-only}}\xspace}

\newcommand{\oursc}{\textcolor{ours-orange}{\textbf{\textsc{Texedo}}}\xspace}
\newcommand{\oracle}{\textcolor{oracle-color}{\textbf{Oracle}}\xspace}

\begin{document}

\title{\ourmethod{}\raisebox{-1pt}{\includegraphics[height=2.0ex]{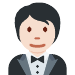}}: Test Time Scaling for Controller-aware Language-conditioned Humanoid Motion Generation}

\author{
    {Jianuo Cao}$^{\ast 1,2}$, {Yuxin Chen}$^{\ast 2}$, {Yuzhen Song}$^{ 2,3}$, {Masayoshi Tomizuka}$^{2}$, {Chenran Li}$^{2}$, {Thomas Tian}$^{2}$\\
    $^1$\textit{Nanjing University} \quad $^2$\textit{University of California, Berkeley} \quad $^3$\textit{Southern University of Science and Technology}
}

\maketitle

\begin{strip}
    \centering
    \vspace{-1.5cm}
    \includegraphics[width=1.0\textwidth]{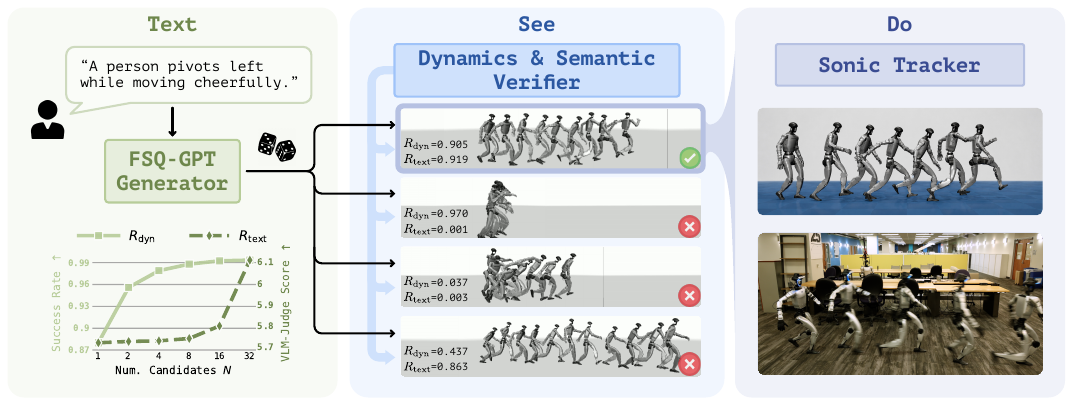}
    \captionof{figure}{\textbf{\ourmethod{} overview.} Given a language prompt, the \textsc{Text} stage samples multiple whole-body motions from a text-conditioned motion generator. The \textsc{See} stage evaluates each candidate with two grounded verifiers: a dynamic score that predicts controller executability and a semantic score that measures text-motion alignment. The \textsc{Do} stage selects a deployable motion by first filtering out dynamically infeasible candidates and then choosing the most semantically aligned one.}
    \label{fig:intro_teaser}
    % \vspace{+0.7em}
\end{strip}

\begin{abstract}
Text-conditioned motion generation has become a promising interface for programming humanoid robots, but current generators are often trained on human motion datasets retargeted to robot morphology.
While such data provides rich semantic  kinematic priors, it does not capture the nuances of the whole-body tracking controller, including balance, contact, actuation limits, and controller-specific failure modes.
Therefore, generated motions can be semantically plausible yet difficult or impossible for the robot to execute.  
We propose \ourmethod{}, a test-time scaling framework for humanoid motion generation that improves motion quality without requiring a stronger underlying generator.
Given a text prompt, our method samples multiple motions from the pre-trained text-conditioned generator and selects the best executable and task-aligned motion. 
The reward model combines a dynamic feasibility verifier, distilled from whole-body tracking rollouts to predict physical executability, with a semantic alignment verifier that measures text-motion alignment in a learned co-embedding space.
Our pipeline treats dynamic feasibility as a hard constraint and semantic alignment as the selection objective within the feasible set.
Across large-scale simulation studies and real-world deployment on a Unitree~G1, we show that our test-time scaling strategy consistently improves both tracking fidelity and text alignment, demonstrating that grounded verification is an effective path toward deployable language-guided humanoid motion generation. Please checkout our website for robot videos, code and data: \href{https://jianuocao.github.io/TEXEDO/}{\textcolor{orange}{Project Website}}.
\end{abstract}

% \begin{IEEEkeywords}
% Humanoid Robot, Motion Generation, Test-Time Scaling
% \end{IEEEkeywords}

\IEEEpeerreviewmaketitle

\section{Introduction}
\label{sec:intro}

Recent progress on language-controlled humanoids has been a story of scaling at the two ends of a single deployment pipeline.
On the \emph{motion generation} side, text-conditioned motion generation models have grown from small task-specific networks into large models that produce whole-body motions directly from a language prompt~\citep{motionGPT,t2mgpt,mdm,momask,motiondiffuse,rempe2026kimodo}. 
On the \emph{whole-body control} side, learning-based whole-body tracking controllers have evolved from task-specific controllers into multi-task controllers that can, in principle, execute essentially any kinematic reference on real hardware~\citep{peng2021amp,phc,beyondmimic,sonic}. 

Despite this progress, these two scaling trends remain only weakly coupled. Most motion-generation models learn from human motion corpora, often after re-targeting the kinematics to a robot skeleton~\citep{amass,humanml3d}, but their training objective does not expose them to the downstream controller that must realize the motion.
As a result, a generated motion may be semantically faithful % and kinematically plausible in isolation,
yet still fall outside the feasible set of the humanoid tracker due to balance constraints, contact timing, actuation limits, or controller-specific failure modes.
In such cases, the burden is shifted to the controller: it must attempt to realize a reference motion that was never generated with its capabilities in mind.
This leads to brittle and poorly specified behavior at deployment, where the executed motion may deviate from the intended semantics, lose balance, produce improper contacts, or saturate actuators.

In language modeling, an increasingly successful way to improve generation at test time is to sample multiple candidates and use a verifier to select the best one~\citep{lightman2023lets,snell2024scaling,openai_o1,deepseekr1}.
Recent work on vision-language-action policies suggests that a similar principle can improve robotic decision-making during deployment~\citep{robomonkey,wu2025foresight}.
However, importing this recipe to humanoid motion generation requires a more grounded notion of what ``best’’ means.
A good candidate motion must not only match the language prompt; it must also be executable by the particular whole-body controller on the particular humanoid platform.
Thus, the verifier cannot merely judge abstract motion quality or text-motion similarity.
It must be \emph{controller-aware}: evaluate whether a candidate trajectory lies within the balance, contact, and actuation budgets of the tracker that will actually execute it.

We propose \oursc{} (\textbf{Text-See-Do}), a test-time scaling pipeline that turns a frozen text-to-motion generator into a controller-aware motion generator    (\Cref{fig:intro_teaser}).
Given a language prompt, \ourmethod{} first samples multiple candidate whole-body trajectories (\textsc{Text}).
It then evaluates each candidate along two complementary axes (\textsc{See}): whether the motion is likely to be executable by the deployment-time tracker, and whether it remains semantically aligned with the language prompt.
The first score is distilled from rollouts of the tracker itself, making it grounded in the balance, contact, and actuation limits of the target humanoid.
The second score is learned from a contrastive text--motion embedding, allowing the selector to preserve the intended semantics of the prompt.
Because dynamic feasibility and semantic alignment are not interchangeable, \ourmethod{} composes them asymmetrically (\textsc{Do}): it first removes motions that are unlikely to be trackable, and then selects the most semantically aligned candidate from the feasible set.
In this way, \ourmethod{} converts additional inference-time samples into a single reference trajectory that is both language-aligned and grounded in the capabilities of the downstream controller, without modifying either the generator or the tracker.

% The contributions of this work are threefold: 1)We formulate test-time scaling for language-conditioned whole-body humanoid motion generation, where multiple candidate motions are sampled from a generator and selected according to both semantic alignment and downstream trackability. 2) We introduce two complementary \emph{grounded} verifiers that score a reference motion using signals unavailable to the generator at selection time: a Dynamic Verifier distilled from whole-body tracker rollouts to estimate whether the motion can be realized by the downstream tracker, and a Semantic Verifier trained directly on robot-skeleton motions to measure text-motion alignment in the robot motion space. 3) We show that semantic alignment and trackability induce competing selection criteria, and propose an asymmetric filter-then-rerank strategy that first enforces dynamic realizability and then reranks the remaining candidates by semantic alignment. Empirically, this strategy improves both tracking success and motion quality, transfers zero-shot to an independently trained generator Kimodo~\citep{rempe2026kimodo}, generalizes to compositional out-of-distribution prompts, and deploys on a physical Unitree~G1 humanoid.

The contributions of this work are threefold.
First, we formulate controller-aware, language-conditioned whole-body humanoid motion generation as a test-time scaling problem, where additional inference-time samples are used to improve the single reference motion ultimately commanded to the robot.
Second, we introduce two complementary grounded verifiers for controller-aware selection: a Dynamic Feasibility Verifier, distilled from rollouts of the deployment-time tracker to predict whether a reference motion can be realized on the target humanoid, and a Semantic Alignment Verifier, trained directly on robot-skeleton motions to measure text–motion alignment.
Third, we show that semantic alignment and dynamic feasibility define competing selection criteria, and propose an asymmetric filter-then-rerank strategy that treats feasibility as a constraint before selecting the most semantically aligned motion.
Empirically, \ourmethod{} improves both motion execution quality and task alignment, transfers zero-shot to unseen motion generators, generalizes to out-of-distribution prompts, and yields consistent gains over vanilla sampling without selection in real-world experiments on a physical Unitree~G1 humanoid.

\section{Problem Formulation}
\label{sec:formulation}
% \vspace{-0.2cm}

We formulate controller-aware whole-body motion generation as a \emph{test-time scaling} problem: instead of committing to the first motion sampled from a generator, we spend inference-time compute to sample multiple candidate motions and select the one most suitable for execution.
Formally, let $\ell$ denote a natural-language instruction and let $\mathbf{m}\in\mathbb{R}^{T\times D}$ denote a whole-body reference motion, where $T$ is the sequence length and $D$ is the dimension of the target humanoid’s root pose and joint configuration.
A text-to-motion generator $\mathcal{G}$ induces a distribution $\mathcal{G}(\cdot\mid \ell)$ over such motions.
At deployment, rather than executing a single sample from this distribution, we draw a candidate set and use a selector to choose a motion that is both dynamically feasible for the downstream tracker and semantically aligned with the instruction.
\section{Method}
\label{sec:method}
% \vspace{-0.2cm}

Given a language instruction $\ell$, \ourmethod{} follows a \textsc{Text--See--Do} runtime selection pipeline. \Cref{sec:text}  \textsc{(Text)}: A language-conditioned generator samples $N$ candidate motions $\{\mathbf{m}_i\}_{i=1}^{N}$ from $\ell$; \Cref{sec:see}  \textsc{(See)}: grounded verifiers score their dynamic feasibility and semantic alignment; and \Cref{sec:do}  \textsc{(Do)}: a filter-then-rerank rule selects one deployable motion. Both the generator and deployment-time tracker remain frozen.

\subsection{Language-Conditioned Motion Generator \textsc{(Text)}}
\label{sec:text}

The \textsc{Text} stage provides the candidate motions that the subsequent \textsc{See} and \textsc{Do} stages operate on. Given a language instruction $\ell$, we assume access to a language-conditioned motion generator $\mathcal{G}$ that can sample a set of candidate reference trajectories
$
    \mathcal{M}_N(\ell)=\{\mathbf{m}_i\}_{i=1}^{N},
    \mathbf{m}_i \sim \mathcal{G}(\cdot \mid \ell).
$
In our framework, $\mathcal{G}$ is treated as a black box: the verifiers consume only the decoded reference trajectories $\{\mathbf{m}_i\}$ and do not rely on the generator architecture.% tokenization scheme, likelihoods, or internal representations.

For the main experiments, we instantiate $\mathcal{G}$ with \basegen{}, a discrete autoregressive text-to-motion generator for whole-body humanoid trajectories.
\basegen{} first compresses continuous robot motions into discrete motion tokens using an FSQ tokenizer~\citep{fsq}, and then fine-tunes a Flan-T5-base model~\citep{flant5} to generate these tokens autoregressively from language instructions.
At inference time, sampled token sequences are decoded back into continuous whole-body reference trajectories, forming the candidate set $\mathcal{M}_N(\ell)$. We sample with temperature and top-$k$ truncation to control candidate diversity. Full training and sample details for \basegen{} are provided in Appendix~\ref{app:generator}.

\subsection{Dynamic Feasibility And Semantic Alignment Evaluation \textsc{(See)}}
\label{sec:see}

\begin{figure}[tbh]
  \centering
  \includegraphics[width=\linewidth]{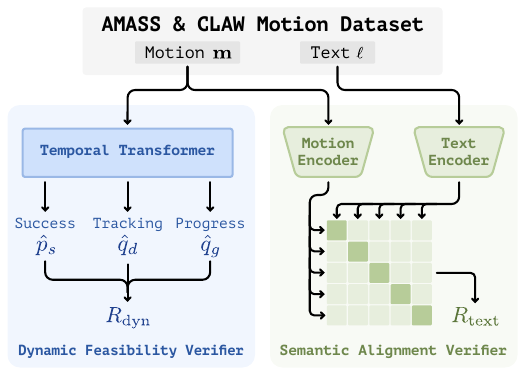}
  \caption{\textbf{Dual verifier design.} $\Rdyn$ estimates dynamic feasibility from the motion alone and $\Rtext$ measures text-motion alignment in a learned embedding space.}
  \label{fig:see_dual}
\end{figure}

What turns a pool of candidate motions into a single deployable trajectory is, ultimately, a definition of \emph{motion quality}. The \textsc{See} stage answers this question with two scalars assigned to every candidate: a dynamic score $\Rdyn(\mathbf{m})\!\in\![0,1]$ measuring whether the motion is dynamically feasible on the deployment-time tracker, and a semantic score $\Rtext(\ell,\mathbf{m})\!\in\![0,1]$ measuring whether it depicts what the prompt describes.
The design of both verifiers is illustrated in \Cref{fig:see_dual}.

\textbf{Dynamic feasibility verifier.}
A kinematically plausible motion may still fail under the deployment-time tracker because of balance, contact, or actuation constraints. However, evaluating each candidate by an actual roll-out at selection time is prohibitive. Thus, we proceed in two steps: % we first define a roll-out level oracle quality $Q^*$ that converts a single offline roll-out into a scalar measure of dynamic feasibility, and we then distil it into a fast motion-only surrogate that scores candidates without invoking the tracker at selection time.

\emph{Oracle quality of roll-outs.}~~For each candidate $\mathbf{m}$, an offline tracker roll-out produces three complementary signals: a binary success indicator $y_s\!\in\!\{0,1\}$, a normalized tracking quality score $q_d\!\in\![0,1]$, and a progress ratio $q_g\!\in\![0,1]$ measuring the percentage of reference completed before early termination. 
We adopt the early-termination criterion of BeyondMimic~\citep{beyondmimic} as the common failure definition for both $y_s$ and $q_g$, and details are reported to Appendix~\ref{app:verifier}. 

To jointly capture these signals, we collapse the three signals into a single oracle quality:
\begin{equation}
  Q^*(y_s, q_d, q_g)
  \;\triangleq\;
  y_s \cdot \frac{1 + \alpha\, q_d}{1 + \alpha}
  \;+\;
  (1 - y_s)\, \beta\, q_g\, q_d.
  % \qquad
  % \alpha = 0.4,\; \beta = 0.6 .
  \label{eq:qstar}
\end{equation}
% The two regimes encode different priorities: $\alpha$ controls how much tracking quality modulates the reward of a \emph{successful} rollout, while $\beta$ sets the partial credit assigned to a \emph{failed} rollout based on its progress and tracking quality. The choice $\beta<1/(1+\alpha)$ guarantees that every successful rollout scores strictly higher than every failed one, so $Q^*$ never trades a feasible candidate for an infeasible one with high partial credit (hyperparameter ablations in \cref{app:dyn_analysis}). 
Here, $\alpha$ controls how much motion-tracking quality modulates the reward for a \emph{successful} roll-out, while $\beta$ sets the partial-credit weight assigned to \emph{failed} roll-outs based on their progress and tracking quality. 
% (ablation of both See \cref{app:dyn_analysis}).
%The formulation of $Q^*$ is designed to jointly capture task completion, motion quality, and partial progress. 
In particular, we set $\beta < 1/(1+\alpha)$, so that $Q^*$ assigns successful roll-outs a higher score than any failed roll-outs, while retaining the graded information within the all success or all failure group. 
% We use $Q^*$ throughout the paper both as the supervision target for the surrogate below and as the reference metric reported in \cref{sec:dyn}.

\emph{Motion-only surrogate.}~~After obtaining the oracle quality labels, we train a lightweight temporal Transformer to predict $(\hat{p}_s,\hat{q}_d,\hat{q}_g)$ from $\mathbf{m}$, and feed the predictions back through \cref{eq:qstar}:
\begin{equation}
  \Rdyn(\mathbf{m})
  \;\triangleq\;
  Q^*\!\bigl(
    \hat{p}_s(\mathbf{m}),\,
    \hat{q}_d(\mathbf{m}),\,
    \hat{q}_g(\mathbf{m})
  \bigr),
  % \;\in\; [0,1],
  \label{eq:rdyn}
\end{equation}
the surrogate and its supervision signal share exactly one algebraic form. The three heads are trained jointly under a weighted sum of a positive-balanced binary cross-entropy on $\hat{p}_s$, an MSE on $\hat{q}_d$, and an MSE on $\hat{q}_g$ masked to failed rollouts. Full architecture and training details are in Appendix ~\ref{app:verifier}.

\textbf{Semantic alignment verifier.}
While dynamic realizability is essential, a motion is only successful if it also realizes the semantic intent of the text instruction. Following the T2M~\citep{humanml3d}, we train bidirectional GRU text and motion encoders, $\varphi_\text{text}$ and $\varphi_\text{motion}$, \emph{directly} on the target robot skeleton,
%---rather than on a human skeleton followed by retargeting---
so that the embedding space is consistent with the candidates produced by the \textsc{Text} stage. With $d_{ij}\!=\!\|\varphi_\text{text}(\ell_i)-\varphi_\text{motion}(\mathbf{m}_j)\|_2$, we optimize an all-pairs margin contrastive objective:
\begin{equation}
  \mathcal{L}_{\mathrm{match}}
  \;=\;
  \frac{1}{B}\sum_i d_{ii}^{2}
  \;+\;
  \frac{1}{B(B-1)}\sum_{i \neq j}
  \bigl[\delta-d_{ij}\bigr]_+^2,
  \qquad \delta=2.0,
  \label{eq:lmatch}
\end{equation}
where $[x]_+\!=\!\max(0,x)$ and $B$ is batch size. The all-pairs term exposes the encoders to $\mathcal{O}(B^2)$ informative negatives per step rather than the single random negative used in the original T2M formulation. At test time, the semantic score is defined as the exponentiated negative
embedding distance:
\begin{equation}
R_{\text{text}}(t, m) \triangleq
\exp\!\big(-\|\phi_{\text{text}}(t) - \phi_{\text{motion}}(m)\|_2\big) \in (0, 1], 
\end{equation} 
so a higher $\Rtext$ indicates better text--motion alignment.
% To avoid the circularity of evaluating a contrastive verifier with its own retrieval metric, we report VLM-as-Judge as the primary semantic metric in \cref{sec:experiments}. 

% and include R@1/R@3 only as a legacy comparison in \cref{app:r1_legacy}.

\subsection{Test-Time Selection \textsc{(Do)}}
\label{sec:do}

This stage converts the verifier scores into a single deployable
reference motion. Given a prompt $\ell$ and candidate motions
$\{\mathbf{m}_i\}_{i=1}^{N}$, we first filter candidates predicted to be
dynamically infeasible and then select the most semantically aligned survivor:
\begin{equation}
\begin{aligned}
  \hat{\mathcal{S}}
  &=
  \bigl\{
    \mathbf{m}_i :
    \Rdyn(\mathbf{m}_i) > \theta
  \bigr\},\\
  \mathbf{m}^*
  &=
  \begin{cases}
    \displaystyle
    \arg\max_{\mathbf{m}_i \in \hat{\mathcal{S}}}
    \Rtext(\ell,\mathbf{m}_i),
    & \hat{\mathcal{S}} \neq \varnothing, \\[6pt]
    \displaystyle
    \arg\max_i \Rdyn(\mathbf{m}_i),
    & \text{otherwise}.
  \end{cases}
\end{aligned}
  \label{eq:do_rule}
\end{equation}
The rule here is deliberate and reflects a deployment-safety priority:
on physical hardware, a semantically perfect but dynamically infeasible motion
can cause falls or actuator saturation, whereas a feasible but slightly
less-aligned motion merely underperforms. We therefore treat dynamic
feasibility as a hard constraint to be satisfied first, and semantic alignment
as the objective to optimize only within the executable pool. We set $\theta=0.8$ from validation roll-outs and the empty-set fallback $\hat{\mathcal{S}} =  \varnothing$ is triggered on only $1.18\%$ of test prompts.

\begin{figure*}[th]
  \centering
  \includegraphics[width=\linewidth]{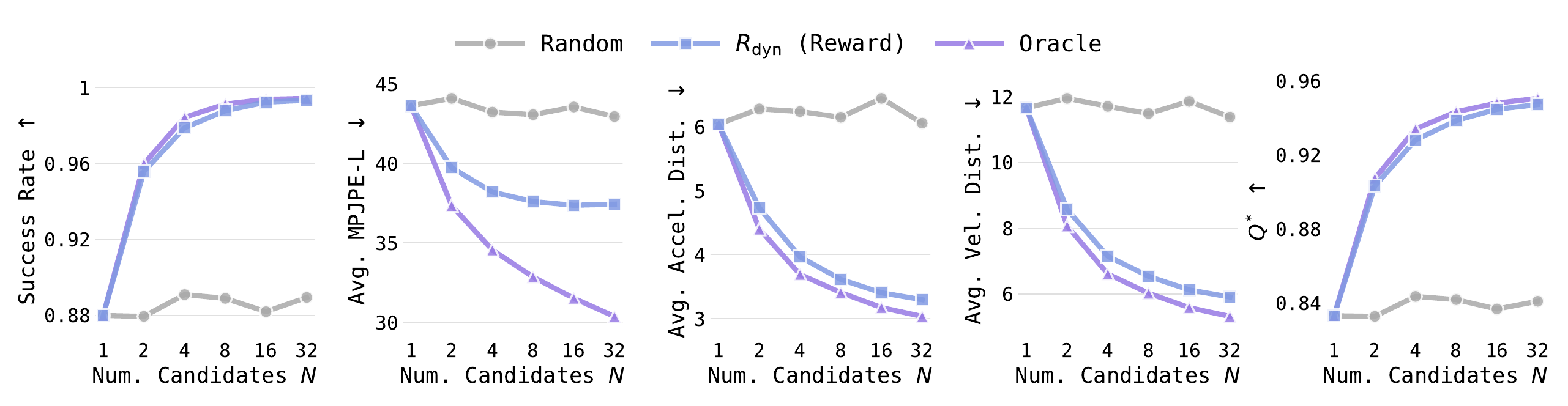}
  \caption{
    \textbf{Best-of-$N$ curves for dynamic feasibility on \basegen{}.}
    Subplots report Succ, $E_\text{mpjpe-l}$, $E_\text{acc}$, $E_\text{vel}$, and $Q^*$ as $N$ increases from $1$ to $32$.
    Across metrics, \rdynonly{} consistently outperforms \randomsel{} and closely approaches \oracle{}.
    % Motions selected by $\Rdyn$ are also comparable to \gtmotion{}, even slightly surpassing it in Succ and $Q^*$.
  }
  \label{fig:bon_main}
\end{figure*}

\section{Experiments}
\label{sec:experiments}

We evaluate \ourmethod{} as a deployable test-time scaling system for language-conditioned humanoid motion generation along three axes: test-time selection performance, zero-shot generalization, and real-world transfer.
\textbf{Q1.} Do the grounded verifiers and their composition select motions that are both executable by the whole-body controller and aligned with the language instruction (\Cref{sec:verifiers})?
\textbf{Q2.} Can the verifiers transfer plug-and-play to an unseen generator and out-of-distribution prompts without retraining? (\Cref{sec:gen})?
\textbf{Q3.} Do the simulation gains carry over to real-world execution on a Unitree~G1 humanoid (\Cref{sec:real_robot})?

\subsection{Experimental Setup}

\textbf{Robot set up and model training.} We use a Unitree~G1 humanoid as our robot platform with SONIC~\citep{sonic} as the frozen whole-body controller. We train our text-to-motion generator using AMASS~\citep{amass} and CLAW~\citep{cao2026clawcomposablelanguageannotatedwholebody}.
More implementation details can be found in Appendix~\ref{app:generator}.

\textbf{Baselines.}
We compare \oursc{} against four selection strategies: 
(i) \base{} ($N{=}1$), a single sample from the generator (no-test-time-scaling);
(ii) \rdynonly{}, $\arg\max_i\Rdyn(\mathbf{m}_i)$, a purely physical selector;
(iii) \rtextonly{}, $\arg\max_i\Rtext(\ell,\mathbf{m}_i)$, a purely semantic selector; and (iv) \oracle{}, a privileged selector that rolls out \emph{every} candidate in $\mathcal{M}_N(\ell)$ through SONIC and picks the one that actually scores best on each metric. Note that \oracle{} requires $N$ SONIC rollouts per prompt (three orders of magnitude more expensive than scoring with the verifiers) and is therefore impractical at deployment time; we report it only as an upper bound on any selector over the same pool. 
%The ground-truth reference motion \gtmotion{} is included as a target ceiling. \clnote{As a generation framework, what is ground-truth reference motion?}

\textbf{Evaluation metrics.}
We evaluate selected motions along two complementary dimensions.
\emph{Dynamic fidelity} measures whether a reference motion can be faithfully executed by the downstream tracker. We report tracking success rate (Succ), root-relative mean per-joint position error ($E_\text{mpjpe}$, mm), joint acceleration error ($E_\text{acc}$, mm/frame$^2$), joint velocity error ($E_\text{vel}$, mm/frame), and the composite oracle quality $Q^*$ defined in \Cref{eq:qstar}.
% \emph{Distributional quality} measures whether the selected motions remain natural and diverse; following~\citep{motionGPT}, we report FID and Diversity (Div)\trnote{Full def of FID?}
\emph{Semantic alignment} measures whether the selected motion matches the language instruction, using VLM-as-Judge as our primary evaluation. Full details of metric definitions and the VLM-based evaluation method are provided in the Appendix~\ref{app:evals}.

\textbf{Inference-time budget.}
All timings are measured on a single NVIDIA A100-SXM4 40GB GPU. Because the $N$
candidates for a prompt are sampled and scored as a single batch, the cost of test-time scaling grows sublinearly in $N$. Sampling  $32$ candidates from FSQ-GPT takes $2.62$\,s on average, and scoring 32 candidates adds only
$15.12$\,ms (Dynamic Feasibility Verifier) and $11.93$\,ms (Semantic Alignment Verifier) on average,
so a full $N{=}32$ sample and selection completes in roughly $3.5$\,s end-to-end. 

\begin{figure}[tbh]
  \centering
  \includegraphics[width=\linewidth]{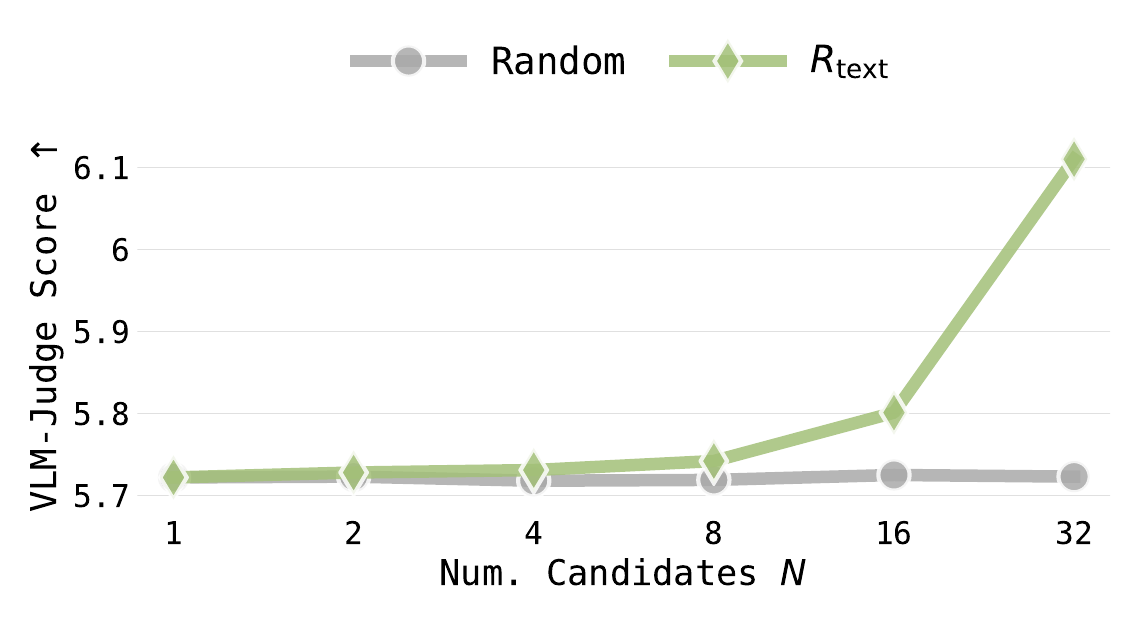}
  \caption{
    \textbf{Semantic alignment verifier turns larger candidate pools into better prompt alignment.}
    For each prompt, \rtextonly{} selects the candidate with the highest $\Rtext$ score; the selected motions improve consistently with increasing $N$ when evaluated by an independent VLM Judge.}
  \label{fig:bon_semantic}
\end{figure}

\begin{table*}[t]
  \centering
  % \vspace{-0.3cm}
  \caption{
    \textbf{\oursc{} balances the complementary strengths of $\Rdyn$ and $\Rtext$.} Pushing either verifier to its extreme costs the other, and \oursc{} achieves strong performance on both axes simultaneously.
}
  \label{tab:main_ablation}
  \normalsize
  %\resizebox{0.7\linewidth}{!}{%
  \begin{tabular}{l c ccccc}
    \toprule
    \multirow{2}{*}{\textbf{Strategy}}
      & \textbf{Semantic alignment}
      & \multicolumn{5}{c}{\textbf{Motion execution quality}} \\
    \cmidrule(lr){2-2}\cmidrule(lr){3-7}
      & VLM-Judge$\uparrow$
      & Succ$\uparrow$ & $E_\text{mpjpe}\downarrow$
      & $E_\text{acc}\downarrow$ & $E_\text{vel}\downarrow$ & $Q^*\uparrow$ \\
    \midrule
    \base{} ($N{=}1$)
      & 5.722 & $0.873$ & $44.34$ & $6.09$ & $11.82$ & $0.829$ \\
    \rdynonly{} ($N{=}32$)
      & 4.924
      & $\textbf{0.990}$ & $\textbf{38.15}$ & $\textbf{3.30}$ & $\textbf{5.91}$ & $\textbf{0.945}$ \\
    \rtextonly{} ($N{=}32$)
      & \textbf{6.110} & $0.885$ & $42.97$ & $4.90$ & $10.03$ & $0.847$ \\
    \oursc{} ($N{=}32$)
      & $\underline{6.054}$
      & $\underline{0.984}$ & $\underline{39.09}$ & $\underline{4.26}$ & $\underline{7.78}$ & $\underline{0.926}$ \\
    \bottomrule
  \end{tabular}
  %}
\end{table*}

\begin{table*}[tb]
  \centering
\caption{
\textbf{Zero-shot transfer to Kimodo.}
\oursc{}, trained only on \basegen{} rollouts, improves both VLM-Judge and physical execution metrics when applied zero-shot to the unseen Kimodo motion generator.
}
  \label{tab:gen_kimodo}
  % \resizebox{0.8\linewidth}{!}{%
  \normalsize
  \begin{tabular}{l c ccccc}
    \toprule
    \multirow{2}{*}{\textbf{Method}}
      & \textbf{Semantic alignment}
      & \multicolumn{5}{c}{\textbf{Motion execution quality}} \\
    \cmidrule(lr){2-2}\cmidrule(lr){3-7}
      & VLM-Judge$\uparrow$
      & Succ$\uparrow$ & $E_\text{mpjpe}\downarrow$
      & $E_\text{acc}\downarrow$ & $E_\text{vel}\downarrow$ & $Q^*\uparrow$ \\
    \midrule
    Kimodo (\base{}, $N{=}1$)
      & 4.823 & $0.937$ & $38.62$ & $1.84$ & $5.33$ & $0.918$ \\
    Kimodo+\rdynonly{} ($N{=}32$)
      & 5.020 & $\textbf{0.955}$ & $\textbf{35.47}$ & $\textbf{1.42}$ & $\textbf{4.28}$ & $\textbf{0.937}$ \\
    Kimodo+\rtextonly{}($N{=}32$)
      & \textbf{5.717} & $0.942$ & $38.02$ & $1.87$ & $5.43$ & $0.919$ \\
    Kimodo+\oursc{} ($N{=}32$)
      & \underline{5.381} & $\underline{0.954}$ & $\underline{37.49}$ & $\underline{1.70}$ & $\underline{4.97}$ & $\underline{0.935}$ \\
    \bottomrule
  \end{tabular}
  %}
\end{table*}

\subsection{Effectiveness Of \oursc{} For Test-Time Scaling}
\label{sec:verifiers}

% Before evaluating the complete \oursc{} selection rule, we first isolate the contribution of each verifier in a best-of-$N$ setting. Specifically, we examine whether increasing the candidate budget allows the Dynamic 
% Verifier and the Semantic Verifier to select motions with improved dynamic feasibility and semantic alignment, respectively.

\textbf{Effectiveness of individual verifiers.}
We first isolate the effect of each verifier under the same best-of-$N$ setting.
For each prompt, we sample a candidate pool of $N$ different samples and select one motion using either the Dynamic Feasibility 
Verifier or the Semantic Alignment Verifier alone.

\Cref{fig:bon_main} reports the dynamic feasibility metrics of the selected motions as the candidate budget $N$ increases, comparing \rdynonly{} against \randomsel{} and \oracle{}.
As $N$ increases, \rdynonly{} consistently outperforms \randomsel{} and approaches \oracle{} across physical metrics, showing that the learned verifier recovers much of the improvement available from oracle-based selection. This is important because \oracle{} requires rolling out every candidate with high-fidelity simulation, making it impractical as an online selection rule.

A similar trend holds for semantic alignment.
When selecting candidates by $\Rtext$, the VLM-as-Judge score (independent third-party judge) increases monotonically with the candidate budget, from $5.722$ at $N{=}1$ to $6.110$ at $N{=}32$ (\Cref{fig:bon_semantic}).

Together, these results show that the two verifiers convert additional test-time compute budgets into improvements along their intended axes: $\Rdyn$ improves executability, while $\Rtext$ improves semantic alignment.

\textbf{Effectiveness of \oursc{}}. We next evaluate the full selector, which composes the two verifiers through an asymmetric filter-then-rerank rule. Given a candidate pool, we first filter out motions predicted to be dynamically infeasible, and then selects the most semantically aligned motion among the remaining candidates.

Table~\ref{tab:main_ablation} compares \oursc{} against \base{} and two
single-verifier selectors (\rdynonly{} and \rtextonly{}) at $N{=}32$.
Optimizing feasibility alone is not neutral to semantics but detrimental:
\rdynonly{} drives the VLM-Judge score down to $4.924$, \emph{below} the
no-selection \base{} ($5.722$), as the most trackable candidates tend to be
conservative motions that drift from the prompt's intent. Symmetrically,
\rtextonly{} reaches $6.110$ in alignment but yields little improvement in
execution (Succ $0.885$ vs.\ \base{}'s $0.873$, far below \rdynonly{}'s $0.990$).
Since pushing either verifier to its extreme costs the other, \ourmethod{} resolves the tension by composition rather than compromise: it recovers execution close
to \rdynonly{} (Succ $0.984$ vs. $0.990$) while largely preserving the semantic alignment of
\rtextonly{} ($6.054$ vs.\ $6.110$).

\textbf{Qualitative comparison.}
Figure~\ref{fig:sim_qual} shows four representative case studies, illustrating different failure modes of the baselines and a successful case of \oursc{}.
\begin{figure*}[ht]
  \centering
  \includegraphics[width=\linewidth]{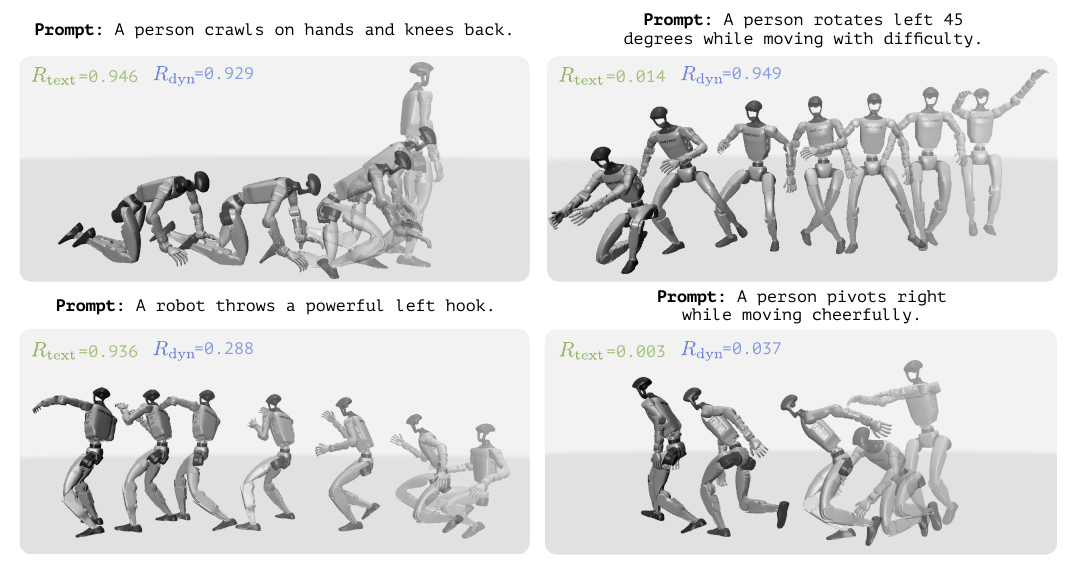}
  \caption{
 \textbf{Qualitative comparison.}
    Representative case studies in MuJoCo simulation.
    \textbf{Top-left:} a success case of \oursc{}, where the generated motion is dynamically feasible and semantically aligned.
    \textbf{Top-right:} a failure case of \rdynonly{} produces a stable but semantically generic motion.
    \textbf{Bottom-left:} a failure case of \rtextonly{} selects a semantically aligned but untrackable motion.
    \textbf{Bottom-right:} a failure case of \base{} fails with a dynamically infeasible and semantically inconsistent motion.
    }
  \label{fig:sim_qual}
\end{figure*}

\begin{table*}[thb]
  \centering
\caption{\textbf{Out-of-distribution generalization.} Evaluated on prompts from unseen dataset BONES-SEED~\citep{bonesstudio2026seed}, \oursc{} achieves a strong balance between dynamics fidelity and semantic alignment.}
  \label{tab:gen_ood}
  \normalsize
  %\resizebox{\linewidth}{!}{%
  \begin{tabular}{ll c c ccccc}
    \toprule
    \multirow{2}{*}{\textbf{Setting}} & \multirow{2}{*}{\textbf{Method}}
      & \textbf{Semantic alignment}
      & \multicolumn{5}{c}{\textbf{Motion execution quality}} \\
    \cmidrule(lr){3-3}\cmidrule(lr){4-8}
      & 
      & VLM-Judge$\uparrow$
      & Succ$\uparrow$
      & $E_\text{mpjpe}\downarrow$
      & $E_\text{acc}\downarrow$
      & $E_\text{vel}\downarrow$
      & $Q^*\uparrow$ \\
    \midrule
   \multirow{2}{*}{ID}
      & \base{}($N{=}1$)
      & 5.722 & 0.873 & 44.34 & 6.09 & 11.82 & 0.829 \\
      & \oursc{} ($N{=}32$)
      & \textbf{6.054} & \textbf{0.984} & \textbf{39.09} & \textbf{4.26} & \textbf{7.78} & \textbf{0.926} \\
    \midrule
    \multirow{2}{*}{OOD}
      & \base{}($N{=}1$)
      & 4.116 & 0.860 & 36.22 & 8.36 & 15.71 & 0.804 \\
      & \oursc{} ($N{=}32$)
      & \textbf{4.401} & \textbf{0.980} & \textbf{26.64} & \textbf{3.65} & \textbf{3.88} & \textbf{0.942} \\
    \bottomrule
  \end{tabular}
  %}
\end{table*}

\begin{figure*}[h]
  \centering
  \includegraphics[width=\linewidth]{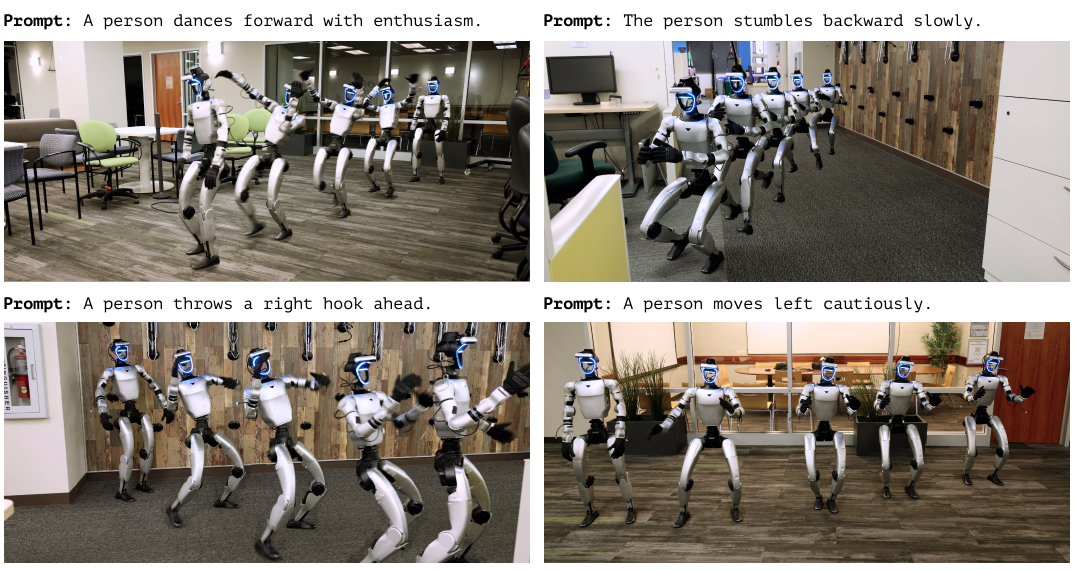}
  \caption{
\textbf{Real-robot execution results on the Unitree~G1.} Across diverse prompt types spanning locomotion, upper-body gestures, and their combinations, TEXEDO consistently selects motions that execute stably and faithfully realize the intended semantics on physical hardware.
    }
  \label{fig:real_robot}
\end{figure*}

\subsection{Out-Of-Distribution Generalization}
\label{sec:gen}

\textbf{Zero-shot transfer to an unseen motion generator.}
A key advantage of \oursc{} is that its verifiers operate on decoded robot motions rather than generator-specific internals.
The Dynamic Feasibility Verifier consumes only the candidate reference motion, and the Semantic Alignment Verifier consumes only the motion and the language instruction.
As a result, the selection rule is not tied to the architecture, tokenization scheme, or latent space of the generator used to produce the candidates.

To test this generator-agnostic property, we apply the same verifiers trained on \basegen{} rollouts directly to Kimodo~\citep{rempe2026kimodo}, an independently trained motion generator zero-shot.
As shown in \Cref{tab:gen_kimodo}, \oursc{} improves both execution quality and semantic alignment on this unseen generator, demonstrating that the learned criteria transfer beyond the generator distribution used to train the verifiers.

\textbf{Out-of-distribution (OOD) generalization.}
We further test whether \oursc{} remains effective on language instructions outside the training distribution.
The base generator is trained on the combined AMASS~\citep{amass} and CLAW~\citep{cao2026clawcomposablelanguageannotatedwholebody} datasets, whereas the OOD evaluation uses 50 prompts sampled from BONES-SEED~\citep{bonesstudio2026seed}.
This setting tests whether the verifier-based selection rule can improve generation when the prompt composition differs from the data used to train the generator and verifiers.
%
% As shown in \cref{tab:gen_ood}, \oursc{} improves over the base generator on both ID and OOD prompts.
% On the OOD split, \oursc{} increases tracking success from $0.860$ to $0.980$ and improves the oracle quality $Q^*$ from $0.804$ to $0.943$, while substantially reducing tracking errors.\jcnote{didn't mention semantic improvement and balance}
As shown in \Cref{tab:gen_ood}, \oursc{} improves over the base generator on both ID and OOD prompts along both axes. On the OOD split, it raises tracking success from $0.860$ to $0.980$ and oracle quality $Q^*$ from $0.804$ to $0.943$ and substantially reducing tracking errors. These results indicate that the feasibility verifier generalizes well beyond the training distribution. In contrast, the semantic gain is smaller than on ID prompts, likely because the learned semantic embedding transfers less readily to unseen prompts.

\subsection{Real-Robot Deployment}
\label{sec:real_robot}
% We deploy \oursc{} on a physical Unitree G1 humanoid across 30 text prompts spanning locomotion, upper-body gestures, and their combinations. Using N=32 candidates per prompt, \oursc{} achieves a tracking success rate of 30/30, compared to XX/30 for the base generator (N=1). Figure 5 shows qualitative results and full deployment details, including prompt construction, hardware setup, and per-prompt tracking results are provided in Appendix~\ref{app:additional_results}
We deploy \oursc on a physical Unitree G1 humanoid across 30 text prompts spanning locomotion, upper-body gestures, and their combinations. Using N=32 candidates per prompt, \oursc{} successfully executes all 30 prompts on hardware. \Cref{fig:real_robot} shows representative qualitative results: \oursc consistently selects motions that execute stably and match the intended behavior. Full deployment details including hardware setup, prompt construction, and per-prompt tracking results are provided in Appendix~\ref{app:robot}.
% \begin{table}[t]
%   \centering
%   \caption{ 
%     }
%   \label{tab:real_robot}
%   \begin{tabular}{l cccc}
%     \toprule
%     \textbf{Method}
%     & Succ$\uparrow$
%       & $E_\text{mpjpe}\downarrow$
%       & $E_\text{acc}\downarrow$
%       & $E_\text{vel}\downarrow$
%        \\
%     \midrule
%     \oursc{} ($N{=}32$)
%       & \textbf{xx.xx}
%       & \textbf{xx.xx}
%       & \textbf{xx.xx}
%       & \textbf{xx} \\
%     \bottomrule
%   \end{tabular}
% \end{table}

\section{Conclusion}
\label{sec:conclusion}

In this work, we presented \ourmethod{}, a test-time scaling framework that turns a frozen language-conditioned motion generator into a controller-aware humanoid motion system. By sampling multiple candidate motions and selecting among them with two complementary verifiers for dynamic feasibility and semantic alignment, \ourmethod{} improves the single reference ultimately commanded to the whole-body tracker without retraining either component.

Our evaluation shows that the verifiers are effective along their intended axes: the Dynamic Feasibility Verifier selects motions that are more executable by the downstream tracker, while the Semantic Alignment Verifier selects motions that better preserve the language instruction. The proposed filter-then-rerank rule combines these complementary signals to select motions that improve both execution quality and semantic alignment. The same verifiers transfer plug-and-play to an unseen generator and out-of-distribution prompts without retraining, and the resulting gains carry over to real-world execution on a Unitree~G1 humanoid. These findings demonstrate the effectiveness of the proposed method for improving deployable language-guided humanoid motion generation.

\section{Limitations}
\label{sec:limitations}

\textbf{Candidate-set and latency trade-off.}
As a test-time selection framework, \ourmethod{} is bounded by the quality and diversity of the sampled candidate motions. If the generator rarely produces motions that are both semantically aligned and dynamically feasible, or if the tracker is too weak to realize any motions, selection alone cannot recover a successful motion. Increasing the number of samples $N$ can improve the chance of finding a valid candidate, but this introduces an inference-time trade-off between selection quality and latency; future work could address this with adaptive sampling mechanisms that allocate additional samples only when needed. 
 
\textbf{Verifier data and tracker adaptation.}
The grounded verifiers need diverse reference motions and tracker roll-outs to learn reliable semantic and dynamic scores. Moreover, because dynamic feasibility is tracker-specific, changing the tracker requires rerunning roll-outs to update oracle feasibility labels and adapt the Dynamic Verifier.  However, the cost of this adaptation process can be effectively mitigated through a data reuse mechanism: the reference-motion corpus is not tied to a particular generator, and even under tracker changes the same references can be reused for relabeling, which is substantially less costly than collecting a new motion corpus or retraining the motion generator. Future work could further reduce this cost by investigating more efficient fine-tuning procedures for verifier adaptation.

\bibliographystyle{plainnat}
\bibliography{references}

\clearpage
\appendices
\section{Related Work}
\label{app:related_work}
% language-conditioned motion generator
\noindent\textbf{Language-conditioned motion generation.}
A prominent line of work represents continuous motion as discrete token sequences. 
MotionGPT \cite{motionGPT} and T2M-GPT \cite{t2mgpt} first learn a VQ-VAE-style motion tokenizer and then train an autoregressive language-model backbone to generate motion tokens from text. 
MoMask \cite{momask} follows the same discrete-token perspective, but uses an RVQ-VAE to construct hierarchical residual motion tokens and generates them through iterative masked prediction rather than left-to-right autoregressive decoding.
In parallel, diffusion-based methods such as MDM~\citep{mdm} and MotionDiffuse~\citep{motiondiffuse} formulate text-to-motion generation as an iterative denoising process, providing an alternative paradigm with strong motion diversity, controllability, and flexible conditioning.
More recently, Kimodo~\citep{rempe2026kimodo} scales kinematic motion diffusion with a large motion-text dataset and supports text prompting together with rich kinematic constraints such as keyframes, waypoints, joint constraints, and foot contacts.
While these works primarily focus on improving the quality and diversity of generated motions, our work addresses the deployment-time trade-off between semantic alignment and downstream tracker compatibility as a plug-in selection module for existing text-conditioned motion generators.

% cotraining generator with traker
\noindent\textbf{Aligning motion generation with tracking.}
To improve the executability of generated motions, many prior works align motion generation and tracking during training by using the same motion corpus for both the generator and the tracker. This encourages distributional alignment between generated references and the tracker's training data~\citep{sonic,tao2026heracles}. Some methods further fine-tune the tracker with RL under randomized commands, terrains, or environment configurations while keeping the generator fixed~\citep{zhang2026learning}. Such training-time alignment can improve executability, but it couples executability to how the generator and tracker are trained together. In contrast, our work keeps the generator and tracker fixed and introduces controller awareness through a lightweight runtime verifier, which can be adapted without retraining the motion generator.

\noindent\textbf{Test-time scaling of robot motion generation models.}
Large-scale pretraining has been the dominant route to capable generative models, but it is computationally expensive and ties each model's behavior to its training distribution. Recently, test-time scaling, generating multiple candidates and using a verifier to select among them, has emerged as a powerful complement in large language models~\citep{lightman2023lets,snell2024scaling,openai_o1,deepseekr1}. Its appeal is that it leverages the diversity already latent in a frozen model to adapt to deployment-time signals (e.g., human preference, task constraints, downstream executability) without retraining the model itself.
The same paradigm has recently begun to be applied to vision--language--action (VLA) policies~\citep{robomonkey,chen2025reimagination,wu2025foresight}, where additional inference-time samples paired with a learned verifier improve manipulation precision; this line of work, however, has so far targeted relatively simple embodiments and arm-only end-effector actions.
In this work, we bring the same generate-then-verify paradigm to language-conditioned humanoid whole-body motion generation, where it differs along three axes. The candidates are whole-body high-dimensional motions rather than just end-effector actions, so verification must operate in a much more complex action space.
Instead of a reward model trained only to evaluate the semantic or kinematic quality of the generated motion, as in recent text-to-motion test-time alignment methods such as ReAlign~\citep{realign}, whose step-aware reward scores candidate motions on semantic alignment and kinematic realism alone, we make the verifier controller-aware by distilling it from whole-body tracker rollouts. This is critical for humanoid whole-body motion: a kinematically plausible reference can still fall outside the tracker's executable envelope due to balance, contact, or actuation limits, so a motion-only reward cannot certify deployability, only a controller-grounded one can.

\section{Language-conditioned Motion Generator}
\label{app:generator}

This appendix details the \basegen{} generator $\mathcal{G}$ used in the \textsc{Text} stage (\Cref{sec:text}): the motion corpus and its preprocessing (\Cref{app:gen_data}), the FSQ tokenizer that discretizes motion (\Cref{app:gen_fsq}), and the language model that generates motion tokens from text (\Cref{app:gen_lm}).

\begin{table*}[h]
  \centering
  \caption{Sample counts of the combined motion corpus after preprocessing.}
  \label{tab:app_dataset_counts}
  \normalsize
  %\resizebox{0.9\columnwidth}{!}{%
  \begin{tabular}{lrrrr}
    \toprule
    \textbf{Corpus} & \textbf{Train} & \textbf{Val} & \textbf{Test} & \textbf{Total} \\
    \midrule
    AMASS (with HumanML3D text) &14148  & 1768  & 1768 & 17684 \\
    CLAW    & 11977 &1497  & 1497 & 14971  \\
    \midrule
    Combined      & 26125  & 3265 & 3265 & 32655 \\
    \bottomrule
  \end{tabular}
  %}
\end{table*}

\begin{table*}[h]
  \centering
  \caption{FSQ reconstruction loss weights $\lambda_c$.}
  \label{tab:app_fsq_loss}
  \normalsize
  % \resizebox{\columnwidth}{!}{%
  \begin{tabular}{lcccccc}
    \toprule
    root pos & root quat & joint pos & root vel & joint vel & root acc & joint acc \\
    \midrule
    $1.0$ & $1.0$ & $2.0$ & $1.0$ & $1.0$ & $0.5$ & $0.5$ \\
    \bottomrule
  \end{tabular}
  %}
\end{table*}

\subsection{Data Construction}
\label{app:gen_data}

\paragraph{Motion representation}
All motions are expressed in a $36$-dimensional Unitree~G1 representation per frame: root position $\mathbf{p}_\text{root}\in\mathbb{R}^3$, root quaternion (wxyz) $\mathbf{q}_\text{root}\in\mathbb{R}^4$, and $29$ joint positions $\boldsymbol{\theta}\in\mathbb{R}^{29}$, sampled at $50$\,Hz.

\paragraph{Corpora and splits}
The generator is trained on a combined corpus of AMASS~\citep{amass}  datasets retargeted to G1 (using HumanML3D captions as the text source) and the CLAW~\citep{cao2026clawcomposablelanguageannotatedwholebody} dataset. The merged dataset is split $8{:}1{:}1$ into train/val/test, and this split is fixed before training any downstream module so that the FSQ tokenizer, the language model, the Dynamic Feasibility Verifier, and the Semantic Alignment Verifier all share the same held-out test motions. Sample counts are reported in \Cref{tab:app_dataset_counts}. Each motion is paired with 2-5 language instructions, resulting in a total of 9,116 held-out test prompts.

\paragraph{Module-specific preprocessing}
Each downstream module normalizes the shared $36$-dim motion differently: the FSQ tokenizer applies per-channel standardization; the Dynamic Feasibility Verifier augments each frame with first- and second-order finite differences and applies z-score normalization with clipping (\Cref{app:dyn_verifier}); and the Semantic Alignment Verifier replaces the root $XY$ position with frame-to-frame velocity to obtain a global-position-invariant embedding (\Cref{app:sem_verifier}).

\subsection{FSQ Tokenizer}
\label{app:gen_fsq}

\paragraph{Architecture}
The tokenizer is a Finite Scalar Quantization (FSQ) motion VAE~\citep{fsq}. A $50$\,Hz, $36$-dim G1 motion is encoded by two stride-$2$ temporal layers (encoder width $512$, depth $3$, LayerNorm, ReLU) into a code sequence of length $L=T/4$. The FSQ quantizer uses the level vector $[3,3,3,3,3,2,2,2,2,2]$, yielding a fixed codebook of $K=\prod_i \ell_i = 7{,}776$ entries. We adopt FSQ over a learned VQ-VAE codebook to avoid codebook collapse on our comparatively small humanoid corpus, which we found important for stable training.

\paragraph{Reconstruction loss}
Inputs are per-channel standardized and supervised with a SmoothL1 reconstruction loss on the pose together with its first- and second-order temporal derivatives:
\begin{equation}
  \mathcal{L}_\text{fsq}
  = \sum_{c\in\mathcal{C}} \lambda_c\,
    \mathrm{SmoothL1}\!\left(\mathbf{m}_c,\hat{\mathbf{m}}_c\right),
\end{equation}
with per-term weights $\lambda_c$ given in \Cref{tab:app_fsq_loss}. Joint positions are weighted most heavily because they directly drive the downstream tracker.

\paragraph{Training}
Optimization and schedule hyperparameters are summarized in \Cref{tab:app_gen_config}; the tokenizer is trained over sliding windows of $100$ frames.

\subsection{Language Model}
\label{app:gen_lm}

\paragraph{Tokenized sequence-to-sequence interface}
We cast text-to-motion as a sequence-to-sequence problem over a shared vocabulary. The frozen FSQ tokenizer maps a continuous motion $\mathbf{m}\in\mathbb{R}^{T\times36}$ to a discrete code sequence $\mathbf{z}=(z_1,\dots,z_L)$ with $L=T/4$ and $z_\ell\in\{1,\dots,K\}$, $K=7{,}776$. We extend the Flan-T5-base~\citep{flant5} vocabulary with these $K$ codes together with three reserved special tokens marking the start, end, and padding of a motion sequence; the end-of-motion token lets the decoder determine the generated motion length at inference time. Text is truncated to $50$ sub-word tokens and motion to $[16,2048]$ frames, i.e.\ $[4,512]$ motion tokens after the tokenizer's $4\times$ temporal downsampling.

\paragraph{Architecture}
% The generator is the encoder--decoder Flan-T5-base ($\sim$220M parameters): the T5 encoder reads the instruction $\ell$ and the autoregressive decoder emits motion tokens conditioned on the encoder states. We initialize from the instruction-tuned Flan-T5 checkpoint rather than training a decoder-only model from scratch, which provides a strong language prior over the text side of the extended vocabulary.
We initialize from an instruction-tuned Flan-T5 ($\sim$220M parameters), whose encoder provides a strong language-understanding prior for conditioning on diverse instructions, while the decoder weights serve as a warm start for autoregressive generation over the extended motion-token vocabulary.

\paragraph{Training objective}
The model is fine-tuned on paired $(\ell,\mathbf{m})$ data with teacher forcing under the standard autoregressive cross-entropy over motion tokens,
\begin{equation}
  \mathcal{L}_\text{lm}
  = -\sum_{\ell=1}^{L}\log p_\theta\!\left(z_\ell \mid z_{<\ell},\,\mathrm{enc}(\ell)\right),
  \label{eq:app_lm_ce}
\end{equation}
where $\mathrm{enc}(\ell)$ denotes the T5 encoder states for the instruction. 

\paragraph{Sampling for candidate generation}
At inference, the $N$ candidates of the \textsc{Text} stage (\Cref{sec:text}) are drawn i.i.d.\ by ancestral sampling, $\mathbf{z}_i\sim p_\theta(\cdot\mid\ell)$, and each token sequence is decoded back to a $36$-dim trajectory by the frozen FSQ decoder. The diversity of the candidate pool is controlled by the sampling temperature $\tau$ together with top-$k$ / nucleus (top-$p$) truncation: a larger $\tau$ widens the feasible and semantic spread that the verifiers select from, at the cost of more low-quality samples. All training and sampling hyperparameters are summarized in \Cref{tab:app_gen_config}.

\begin{table*}[h]
  \centering
  \caption{Training and sampling configuration of the \basegen{} generator.}
  \label{tab:app_gen_config}
  \normalsize
  % \setlength{\tabcolsep}{4pt}
  %\resizebox{\columnwidth}{!}{%
  \begin{tabular}{lcc}
    \toprule
     & \textbf{FSQ Tokenizer} & \textbf{Language Model} \\
    \midrule
    Backbone / init        & temporal-conv VAE                 & Flan-T5-base (instr.-tuned) \\
    Optimizer              & AdamW                             & AdamW \\
    Learning rate          & $4\times10^{-5}$                  & $8\times10^{-4}$ \\
    $(\beta_1,\beta_2)$    & $(0.9,0.99)$                      & $(0.9,0.99)$ \\
    Weight decay           & $0$                               & $0$ \\
    LR schedule            & cosine                            & cosine \\
    Batch size             & $256$                             & $32$ \\
    Epochs                 & $100$                             & early-stop (val.) \\
    % Max.\ text length      & ---                               & $50$ tokens \\
    % Vocabulary addition    & ---                               & $K{+}3$ motion tokens \\
    Sampling               & ---                               & temp.\ $1.0$, top-$k{=}50$, top-$p{=}0.95$ \\
    \bottomrule
  \end{tabular}
  %}
\end{table*}

\begin{table*}[h]
  \centering
  \caption{Prediction accuracy of the three Dynamic Verifier heads on the held-out test set.}
  \label{tab:head_accuracy}
  \normalsize
  \begin{tabular}{l l r}
    \toprule
    \textbf{Head} & \textbf{Metric} & \textbf{Value} \\
    \midrule
    \multirow{3}{*}{$\hat{p}_s$ (success)}
      & AUROC$\uparrow$       & $0.979$ \\
      & AUPRC$\uparrow$       & $0.997$ \\
      & Fail recall$\uparrow$ & $0.930$ \\
    \midrule
    \multirow{2}{*}{$\hat{q}_d$ (dynamics)}
      & Spearman $\rho$ vs.\ $E_\text{acc}$$\downarrow$ & $-0.936$ \\
      & Spearman $\rho$ vs.\ $E_\text{vel}$$\downarrow$ & $-0.932$ \\
    \midrule
    \multirow{2}{*}{$\hat{q}_g$ (progress)}
      & Spearman $\rho$ vs.\ progress ratio$\uparrow$         & $0.238$ \\
      & Spearman $\rho$ vs.\ progress (on failures)$\uparrow$ & $0.925$ \\
    \bottomrule
  \end{tabular}
\end{table*}

\section{Verifier Training}
\label{app:verifier}

\subsection{Dynamic Feasibility Verifier}
\label{app:dyn_verifier}

\paragraph{SONIC roll-outs and oracle labels}
Oracle feasibility labels are collected by rolling out every reference motion offline through the SONIC whole-body tracker~\citep{sonic} in MuJoCo. Each episode is initialized from the first reference frame and advances until the clip ends or an early-termination condition fires. We adopt the height-based termination criteria of BeyondMimic~\citep{beyondmimic}; to avoid duplication, the precise thresholds, the success flag $y_s$, and the progress ratio $q_g$ are defined once in \Cref{app:evals} and reused verbatim here.
Each roll-out thus yields the three signals consumed by \Cref{eq:qstar}: the binary success flag $y_s$, the progress ratio $q_g$, and a normalized tracking quality
\begin{equation}
  q_d(\mathbf{m})
  = \tfrac{1}{2}\!\left[
      \operatorname{clip}\!\left(1 - \tfrac{e_\text{acc}}{e_\text{acc}^{95}},0,1\right)
    + \operatorname{clip}\!\left(1 - \tfrac{e_\text{vel}}{e_\text{vel}^{95}},0,1\right)
  \right],
\end{equation}
where $e_\text{acc}$ and $e_\text{vel}$ are the acceleration and velocity tracking errors and $e^{95}$ denotes their $95$th-percentile normalizers.

\paragraph{Ordering guarantee}
The constraint $\beta<1/(1+\alpha)$ makes $Q^*$ \emph{feasibility-first}: any successful roll-out scores above any failed one, regardless of other metrics.
\begin{proposition}[Feasibility-first ordering]
\label{prop:ordering}
For $\alpha,\beta>0$, $y_s\in\{0,1\}$ and $q_d,q_g\in[0,1]$, the quality $Q^*$ of \Cref{eq:qstar} satisfies $Q^*(1,\cdot,\cdot)>Q^*(0,\cdot,\cdot)$ for all arguments if and only if $\beta<\tfrac{1}{1+\alpha}$.
\end{proposition}

\begin{proof}
With $q_d,q_g\in[0,1]$, the two regimes of \Cref{eq:qstar} are bounded as
\begin{equation*}
  \begin{aligned}
  Q^*(1,q_d,q_g)
  &= \frac{1+\alpha q_d}{1+\alpha}
  \in\Bigl[\tfrac{1}{1+\alpha},\,1\Bigr],\\
  Q^*(0,q_d,q_g)
  &= \beta\,q_g\,q_d
  \in[0,\,\beta],
  \end{aligned}
\end{equation*}
and both lower/upper endpoints are attained ($q_d=0$ and $q_d=q_g=1$, respectively). The successful range therefore lies strictly above the failed range iff its lower end exceeds the failed upper end, i.e.\ $\tfrac{1}{1+\alpha}>\beta$.
\end{proof}

\noindent The same bounds show $Q^*$ stays graded \emph{within} each group: it increases in $q_d$ among successes and in $q_g q_d$ among failures, so the ``least-bad'' candidate is still preferred when all candidates fail (\Cref{tab:reward_ablation}).

\begin{table}[h]
  \centering
  \caption{\textbf{$\Rdyn$ ranks candidates in close agreement with the simulator oracle without ever running it.} Within-prompt Kendall $\tau$ between $\Rdyn$ and the oracle quality $Q^*$ on 9{,}116 held-out prompts.}
  \label{tab:rank_quality}
  \normalsize
  % \setlength{\tabcolsep}{8pt}
  % \resizebox{0.92\linewidth}{!}{%
  \begin{tabular}{lrr}
    \toprule
    \textbf{Prompt type} & \textbf{\#Prompts} & \textbf{Kendall $\tau\uparrow$} \\
    \midrule
    Mixed (success + failure) & 6{,}759 & $0.674$ \\
    All-failure               &      46 & $0.357$ \\
    All-success               & 2{,}311 & $0.608$ \\
    \midrule
    Overall                   & 9{,}116 & $\textbf{0.656}$ \\
    \bottomrule
  \end{tabular}
  % }
\end{table}

\begin{table*}[h]
  \centering
  \caption{Reward-design ablation at $N{=}8$. Model weights are fixed; only the inference-time reward formula varies. $^\dagger$mean progress on all-failure prompts.}
  \label{tab:reward_ablation}
  \normalsize
  % \setlength{\tabcolsep}{8pt}
  % \resizebox{0.92\linewidth}{!}{%
    \begin{tabular}{clccc}
      \toprule
      & \textbf{Reward formula}
        & \textbf{Succ$\uparrow$}
        & $\bar{Q}^*\uparrow$
        & \textbf{All-fail prog.$^\dagger\uparrow$} \\
      \midrule
      (a) & $\hat{p}_s$ only                 & $\textbf{0.984}$ & $0.922$ & $0.395$ \\
      (b) & $\hat{q}_d$ only                 & $0.953$ & $0.916$ & $0.416$ \\
      (c) & $\hat{q}_g$ only                 & $0.982$ & $0.921$ & $\textbf{0.427}$ \\
      (d) & $\hat{p}_s \cdot \hat{q}_d$      & $0.981$ & $\textbf{0.931}$ & $0.411$ \\
      (e) & $\hat{p}_s \cdot \hat{q}_g$      & $\textbf{0.984}$ & $0.922$ & $0.395$ \\
      (f) & $Q^*$ formula, \Cref{eq:rdyn} (ours) & $\underline{0.983}$ & $\textbf{0.931}$ & $\underline{0.425}$ \\
      \bottomrule
    \end{tabular}
  % }
\end{table*}
\paragraph{Model structure and training strategy}
Each raw $36$-dim frame is expanded into a $94$-dim feature (root dynamics $7$, joint positions $29$, velocities $29$, accelerations $29$), z-score normalized and clipped to $[-10,10]$. An input projection encodes each semantic group to $128$-dim and fuses them into a $256$-dim token; a $4$-layer causal Transformer ($d_{\text{model}}=256$, $4$ heads, pre-LayerNorm) followed by mean-attention pooling and three lightweight MLP heads outputs $(\hat{p}_s,\hat{q}_d,\hat{q}_g)$, which are recombined through \Cref{eq:rdyn} with $\alpha=0.4,\ \beta=0.6$.
The three heads are trained jointly by minimizing, over each batch $\mathcal{B}$,
\begin{equation}
\begin{aligned}
  \mathcal{L}
  &= \underbrace{\mathrm{BCE}_{w^+}\!\bigl(\hat{p}_s,\,y_s\bigr)}_{\text{success}}
  \;+\; \lambda_d\,\underbrace{\bigl(\hat{q}_d-q_d\bigr)^2}_{\text{dynamics}}
  \\[-1pt]
  &\quad+\; \lambda_g\,\underbrace{\frac{\sum_{i\in\mathcal{B}}(1-y_s^i)\bigl(\hat{q}_g^i-q_g^i\bigr)^2}{\sum_{i\in\mathcal{B}}(1-y_s^i)}}_{\substack{\text{progress}\\\text{(failed roll-outs only)}}},
\end{aligned}
  \label{eq:verifier_loss}
\end{equation}
where $\lambda_d=0.6$, $\lambda_g=0.8$, $\mathrm{BCE}_{w^+}$ is a binary cross-entropy with positive-class weight $w^+$ to offset class imbalance, and the first two terms are averaged over $\mathcal{B}$. The factor $(1-y_s)$ masks the progress loss to \emph{failed} roll-outs: a success completes the reference, so $q_g\!\equiv\!1$ is a constant copy of $y_s$ that carries no gradient, whereas on failures $q_g$ varies and is exactly where accurate progress estimates are needed for partial-credit ranking, hence the head is weak on progress overall but strong on failures (\Cref{tab:head_accuracy}). 
\paragraph{Fidelity studies}
We validate $\Rdyn$ on held-out data along two axes: per-head prediction accuracy (\Cref{tab:head_accuracy}) and end-to-end ranking agreement with the oracle (\Cref{tab:rank_quality}).
The success head $\hat{p}_s$ is highly discriminative (AUROC $0.979$, AUPRC $0.997$) and maintains $0.930$ recall on the minority failure class, so infeasible candidates are rarely mislabeled as feasible. The dynamics head $\hat{q}_d$ is strongly rank-correlated with both tracking-error signals (Spearman $\rho=-0.936$ and $-0.932$ against acceleration and velocity errors, negative by construction). The progress head $\hat{q}_g$ has a weak global correlation with the progress ratio ($\rho=0.238$) because most successful roll-outs cluster near $q_g\!=\!1$, but becomes highly informative on the failed roll-outs where progress genuinely varies ($\rho=0.925$), precisely the regime where partial-credit ranking is needed.
At the composite level, the within-prompt Kendall $\tau$ between $\Rdyn$ and $Q^*$ over $9{,}116$ held-out prompts is $0.656$ overall, rising to $0.674$ on \emph{mixed} prompts containing both feasible and infeasible candidates (\Cref{tab:rank_quality}), confirming that $\Rdyn$ recovers the bulk of the oracle ordering without invoking the tracker.

\paragraph{Ablation studies}
We ablate the two design choices that turn the three heads into a single selection reward: the functional \emph{form} of the recombination, and the \emph{weights} it places on each term.
\Cref{tab:reward_ablation} varies the inference-time reward formula at $N{=}8$ with the model weights held fixed. Using $\hat{p}_s$ alone yields the highest raw success rate but collapses to near-random ordering once every candidate in a prompt fails, because it provides no signal to separate equally-infeasible motions; conversely, $\hat{q}_d$ alone breaks the success-first hierarchy and lets smooth-but-failing candidates win. The full $Q^*$ formula of \Cref{eq:rdyn} strikes the best balance: it matches the strongest variants on global quality ($\bar{Q}^*=0.931$) while retaining a high all-failure progress score ($0.425$), so it still selects the ``least bad'' candidate when no motion succeeds.

\subsection{Semantic Alignment Verifier}
\label{app:sem_verifier}

\paragraph{Model structure and training strategy}
The semantic verifier follows the T2M co-embedding design~\citep{humanml3d} but is trained \emph{directly} on the G1 skeleton so that its embedding space is consistent with the candidates produced by the \textsc{Text} stage. The motion encoder applies two temporal Conv1d down-sampling layers (a MovementConvEncoder) followed by a BiGRU to produce a $512$-dim motion embedding $\varphi_\text{motion}(\mathbf{m})$; the text encoder sums $300$-dim word and $15$-dim part-of-speech embeddings and feeds them through a BiGRU to obtain $\varphi_\text{text}(\ell)$. The ConvEncoder is first pretrained as a motion autoencoder and frozen, after which the two BiGRU encoders are trained jointly with the all-pairs margin contrastive objective in \Cref{eq:lmatch} ($\delta=2.0$). The test-time score $\Rtext$ is the exponentiated negative embedding distance.

\paragraph{Fidelity studies}
We verify that the contrastive embedding behind $\Rtext$ has actually learned a separating geometry. On the $9{,}116$ held-out pairs and a $32$-distractor pool, $\Rtext$ attains R@1 $=0.747$ and R@3 $=0.935$ in motion-to-text retrieval, far above the $1/32$ random baseline (\Cref{tab:rtext_sanity}); shuffling the motion input collapses retrieval to near-random, confirming that the score captures text--motion correspondence rather than a marginal motion prior.

\begin{table}[h]
  \centering
  \caption{\textbf{$\Rtext$ separates paired text and motion.} Retrieval sanity check on the held-out test set with a $32$-distractor pool. Shuffling the motion input collapses every metric to near-random.}
  \label{tab:rtext_sanity}
  \normalsize
  % \setlength{\tabcolsep}{8pt}
  % \resizebox{0.92\linewidth}{!}{%
  \begin{tabular}{l cc}
    \toprule
    \textbf{Metric} & \textbf{Paired (test)} & \textbf{Shuffled (control)} \\
    \midrule
    R@1$\uparrow$                     & $\textbf{0.747}$ & $0.014$ \\
    R@2$\uparrow$                     & $\textbf{0.885}$ & $0.059$ \\
    R@3$\uparrow$                     & $\textbf{0.935}$ & $0.102$ \\
    Matching score$\downarrow$        & $\textbf{0.248}$ & $2.977$ \\
    Gap (neg$-$pos)$\uparrow$         & $\textbf{2.835}$ & $0.019$ \\
    \bottomrule
  \end{tabular}
  % }
\end{table}

% \begin{table}[h]
%   \centering
%   \caption{Semantic verifier fidelity: agreement between $\Rtext$ ranking and VLM-as-Judge.}
%   \label{tab:app_sem_fidelity}
%   \small
%   \begin{tabular}{lc}
%     \toprule
%     \textbf{Metric} & \textbf{Value} \\
%     \midrule
%     Within-prompt Kendall $\tau$ ($\Rtext$ vs.\ VLM-Judge)$\uparrow$ & \jcnote{X.XX} \\
%     Retrieval R@1 / R@3 (legacy)$\uparrow$ & \jcnote{X.XX} / \jcnote{X.XX} \\
%     \bottomrule
%   \end{tabular}
% \end{table}

\section{Real-World Results}
\paragraph{Hardware setup and motion collection}
All real-world experiments are conducted on a Unitree~G1 humanoid robot using the SONIC whole-body tracking policy~\citep{sonic}.  We evaluate $30$ text prompts spanning locomotion, upper-body gestures, and their combinations. These motions are used to assess whether the references selected by \oursc{} can be executed reliably on physical hardware.
\paragraph{Results}
\oursc{} successfully executes all 30 real-world trajectories on the physical robot. \Cref{fig:real_robot} in the main text shows representative execution sequences, and \Cref{tab:real_world_tracking_metrics} reports the per-prompt real-world tracking metrics.
\label{app:robot}

\begin{table*}[t]
  \centering
  \caption{Per-prompt real-world deployment metrics. A checkmark indicates that the robot completes the trajectory without falling or early termination. Tracking errors are reported in mm, mm/frame, and mm/frame$^2$, respectively; the overall row reports frame-weighted averages.}
  \label{tab:real_world_tracking_metrics}
  \small
  \setlength{\tabcolsep}{3pt}
  \begin{tabular}{p{0.7\textwidth}cccc}
    \toprule
    \textbf{Prompt} & \textbf{Success} & $E_{\text{mpjpe-l}}\downarrow$ & $E_{\text{vel}}\downarrow$ & $E_{\text{accel}}\downarrow$ \\
    \midrule
    the man is doing salsa dance & \checkmark & 39.14 & 5.60 & 3.60 \\
    the person took a big step over something. & \checkmark & 26.18 & 2.84 & 1.90 \\
    a person walks in a circle clockwise. & \checkmark & 23.11 & 3.13 & 2.42 \\
    A person appears happy while throwing punches and squatting forward. & \checkmark & 47.38 & 4.68 & 2.59 \\
    a person walks forward slowly. & \checkmark & 20.04 & 1.97 & 1.35 \\
    The person holds an object and moves forward. & \checkmark & 29.35 & 3.43 & 1.72 \\
    A person is carrying something heavy while moving forward. & \checkmark & 42.22 & 3.64 & 2.21 \\
    A person throws a right hook ahead. & \checkmark & 40.76 & 5.46 & 3.08 \\
    A person dances happily ahead. & \checkmark & 32.18 & 4.90 & 2.83 \\
    a person walks straight forward & \checkmark & 17.93 & 2.12 & 1.55 \\
    A person executes a left jab forward with force and then throws a left hook ahead with full force. & \checkmark & 46.58 & 6.27 & 3.40 \\
    A person moves cautiously right. & \checkmark & 21.70 & 2.33 & 1.44 \\
    The person moves while boxing and walking forward. & \checkmark & 32.67 & 4.26 & 2.51 \\
    A person jabs right forward and performs a happy dance. & \checkmark & 41.41 & 4.30 & 2.50 \\
    A person advances while boxing ahead aggressively. & \checkmark & 38.68 & 4.31 & 2.20 \\
    the person is stretching, with their arms above their head. & \checkmark & 34.59 & 0.90 & 0.66 \\
    A person carries an object sideways to the right balancing carefully. & \checkmark & 43.68 & 2.99 & 1.93 \\
    A person throws a left hook ahead with full force. & \checkmark & 41.04 & 6.70 & 4.00 \\
    spinning in a circle in the air. & \checkmark & 48.91 & 11.41 & 11.09 \\
    A person shambles left lifelessly. & \checkmark & 41.55 & 4.66 & 2.60 \\
    a man is doing jumping jacks. & \checkmark & 33.68 & 6.12 & 4.15 \\
    a person walking forward aggressively & \checkmark & 21.24 & 2.37 & 1.92 \\
    A person dances joyfully forward with enthusiasm. & \checkmark & 30.10 & 4.76 & 2.75 \\
    A person pivots right while moving cheerfully. & \checkmark & 49.25 & 5.98 & 3.38 \\
    A person rotates right while sneaking. & \checkmark & 33.77 & 4.77 & 2.56 \\
    A person moves while crouching rearward low to the ground. & \checkmark & 58.24 & 4.49 & 2.08 \\
    A person bounces along sideways to the left cheerfully and then scurries away rearward frantically. & \checkmark & 39.62 & 2.75 & 1.71 \\
    A person scurries away sideways to the left in a panic. & \checkmark & 36.96 & 3.83 & 2.32 \\
    a man raises his right hand to his head and then returns it to his side & \checkmark & 23.47 & 1.82 & 1.43 \\
    The person stumbles backward in a slow, zombie-like shamble. & \checkmark & 36.30 & 3.36 & 2.09 \\
    \midrule
    \textbf{Overall} & \textbf{30/30} & \textbf{37.41} & \textbf{4.35} & \textbf{2.67} \\
    \bottomrule
  \end{tabular}
\end{table*}

% real-robot prompt
% real robot results

% \section{Metrics}
% \label{app:evals}

% We report two families of metrics. Physical metrics measure how well the deployment-time tracker executes the selected motion; the semantic metric measures whether the motion realizes the text instruction.

% \paragraph{Physical metrics.}
% All physical metrics aggregate offline SONIC roll-outs under the early-termination criterion of \Cref{app:dyn_verifier}.
% \begin{itemize}[leftmargin=1.4em,itemsep=1pt,topsep=2pt]
%   \item \textbf{Success} (Succ$\uparrow$): fraction of roll-outs that reach the end of the reference without triggering any early-termination condition.
%   \item $E_\text{mpjpe-l}$ ($\downarrow$, mm): root-relative (pelvis-anchored) mean per-joint position error, averaged over bodies.
%   \item $E_\text{accel}$ (accel-dist, $\downarrow$, mm/frame$^2$): error of the second-order finite difference of body positions, averaged over bodies.
%   \item $E_\text{vel}$ (vel-dist, $\downarrow$, mm/frame): error of the first-order finite difference of body positions, averaged over bodies.
%   \item $Q^*$ ($\uparrow$): the oracle quality of \Cref{eq:qstar} combining success, tracking quality, and progress into a single scalar.
% \end{itemize}
% Following the convention in the experiments, $E_\text{mpjpe-l}$, $E_\text{accel}$, and $E_\text{vel}$ are computed over successful roll-outs only, so that tracking error is not confounded with early termination.

\section{Metrics}
\label{app:evals}

We report two families of metrics. Physical metrics measure how well the deployment-time tracker executes the selected motion; the semantic metric measures whether the motion realizes the text instruction. 

\paragraph{Notation}
A roll-out tracks a reference of $T$ frames over $J$ rigid bodies. At frame $t$ let $\mathbf{p}^{\text{ref},j}_t,\mathbf{p}^{\text{rob},j}_t\in\mathbb{R}^3$ be the world position of body $j$ in the reference and in the realized robot state, and let $\mathbf{q}^{\text{ref},0}_t,\mathbf{q}^{\text{rob},0}_t$ be the orientation of the anchor body (pelvis, $j{=}0$). We write $[\cdot]_z$ for the vertical component and $\mathbf{R}(\mathbf{q})\in\mathrm{SO}(3)$ for the rotation matrix of a quaternion.

\paragraph{Termination criteria}
Following BeyondMimic~\citep{beyondmimic}, the episode is terminated at the first frame $t$ at which any of the following \emph{height-based} conditions fires:
\begin{align}
  \text{(anchor height)}\quad
    & \bigl|[\mathbf{p}^{\text{ref},0}_t]_z-[\mathbf{p}^{\text{rob},0}_t]_z\bigr| > h_a, \label{eq:term_anchor_pos}\\
  \text{(anchor tilt)}\quad
    & \Bigl|\bigl[\mathbf{R}(\mathbf{q}^{\text{ref},0}_t)^{\!\top}\mathbf{g}\bigr]_z-\bigl[\mathbf{R}(\mathbf{q}^{\text{rob},0}_t)^{\!\top}\mathbf{g}\bigr]_z\Bigr| > c_o, \label{eq:term_anchor_ori}\\
  \text{(end-effector height)}\quad
    & \max_{b\in\mathcal{B}_\text{ee}}\bigl|[\tilde{\mathbf{p}}^{\text{ref},b}_t]_z-[\tilde{\mathbf{p}}^{\text{rob},b}_t]_z\bigr| > h_e, \label{eq:term_ee}
\end{align}
where $\mathbf{g}=(0,0,-1)^{\!\top}$ is the gravity direction, $\mathcal{B}_\text{ee}$ is the set of end-effector bodies, $\tilde{\mathbf{p}}$ denotes anchor-relative body positions, and $(h_a,c_o,h_e)=(0.25,0.8,0.25)$. Let $\tau$ be the first frame at which \Cref{eq:term_anchor_pos,eq:term_anchor_ori,eq:term_ee} fire, or $\tau=T$ if the roll-out reaches the end of the reference. \Cref{eq:term_anchor_ori} compares the $z$-component of gravity projected into the reference and robot anchor frames, i.e.\ the deviation in anchor tilt.
\paragraph{Success and progress}
\begin{equation}
  \text{Succ} \;=\; \mathbb{1}[\tau = T],
  \qquad
  q_g \;=\; \frac{\tau}{T} \in (0,1],
\end{equation}
so a roll-out is \textbf{successful} ($\uparrow$) only if it reaches the final reference frame without any termination, and the progress ratio $q_g$ is the fraction of the reference completed before termination.

\paragraph{Tracking errors}
With the position finite differences $\mathbf{v}^{\cdot,j}_t=\mathbf{p}^{\cdot,j}_t-\mathbf{p}^{\cdot,j}_{t-1}$ and $\mathbf{a}^{\cdot,j}_t=\mathbf{v}^{\cdot,j}_t-\mathbf{v}^{\cdot,j}_{t-1}$, and the anchor-removed positions $\tilde{\mathbf{p}}^{\cdot,j}_t=\mathbf{p}^{\cdot,j}_t-\mathbf{p}^{\cdot,0}_t$, the tracking errors (in mm, mm/frame, and mm/frame$^2$) are
\begin{align}
  E_\text{mpjpe-l} &= \frac{10^3}{T}\sum_{t=1}^{T}\frac{1}{J}\sum_{j=1}^{J}
    \bigl\|\tilde{\mathbf{p}}^{\text{ref},j}_t-\tilde{\mathbf{p}}^{\text{rob},j}_t\bigr\|_2, \label{eq:mpjpe}\\
  E_\text{vel} &= \frac{10^3}{T-1}\sum_{t=2}^{T}\frac{1}{J}\sum_{j=1}^{J}
    \bigl\|\mathbf{v}^{\text{ref},j}_t-\mathbf{v}^{\text{rob},j}_t\bigr\|_2, \label{eq:vel}\\
  E_\text{accel} &= \frac{10^3}{T-2}\sum_{t=3}^{T}\frac{1}{J}\sum_{j=1}^{J}
    \bigl\|\mathbf{a}^{\text{ref},j}_t-\mathbf{a}^{\text{rob},j}_t\bigr\|_2. \label{eq:accel}
\end{align}
$E_\text{mpjpe-l}$ ($\downarrow$) is the root-relative (pelvis-anchored) mean per-body position error; $E_\text{vel}$ ($\downarrow$) and $E_\text{accel}$ ($\downarrow$) penalize first- and second-order kinematic mismatch and the factor $10^3$ converts m to mm. As in the main experiments, $E_\text{mpjpe-l}$, $E_\text{vel}$, and $E_\text{accel}$ are averaged over the executed frames of \emph{successful} roll-outs only, so that tracking error is not confounded with early termination. Finally, $Q^*$ ($\uparrow$, \Cref{eq:qstar}) collapses success, tracking quality, and progress into the single oracle quality used as the Dynamic Verifier supervision target.

\paragraph{Semantic metric}

% %
% The ensemble varies along two axes. First, we query two VLMs from different families---GPT-5.5 and \texttt{gemini-2.5-flash}---and average their scores, so that the metric does not hinge on a single model's idiosyncrasies. Second, each model scores the pair under a $2{\times}2$ grid of \emph{rubric} and \emph{frame count}. The two rubrics are a \emph{holistic} rubric, which tags each semantic unit (ACTION, BODY-PART, SPATIAL, TEMPORAL, ATTRIBUTE) as MATCH / PARTIAL / MISMATCH and maps the matched fraction to a $1$--$10$ band, and a \emph{per-axis} rubric, which scores the same five axes on a $\{0,1,2\}$ scale and sums them. Each rubric is run at $8$ and $16$ uniformly-sampled frames; the denser sampling catches temporal failures (repetition count, pace, trajectory reversal) that are aliased at $8$ frames. Both rubrics share an \emph{agent-equivalence} clause (the humanoid robot stands in for any human referent in the prompt) and \emph{floor caps} that prevent a single critical failure---e.g., a wrong primary action or opposite-direction translation---from being masked by matched filler. The reported semantic score is the mean over all eight calls (two models $\times$ two rubrics $\times$ two frame counts). Full prompts are in \Cref{app:eval_prompts}.
% \paragraph{Semantic metric.}
To obtain an independent, scalable measure of how well a rendered rollout realizes its text prompt, we adopt a VLM-as-judge ensemble that scores each (text, rendered-rollout video) pair with several low-temperature VLM calls and reports their mean as a single score in $[1,10]$.
The ensemble varies along two main axes to ensure robustness. First, to mitigate the risk of the metric inheriting the idiosyncratic biases of any single model, we query two distinct VLMs from different families, namely GPT-5.5 and \texttt{gemini-2.5-flash}, and average their outputs. Second, each model evaluates the video under a $2 \times 2$ grid combining two evaluation rubrics and two frame counts. The first rubric is a holistic rubric, which tags five key semantic units (ACTION, BODY-PART, SPATIAL, TEMPORAL, ATTRIBUTE) as MATCH, PARTIAL, or MISMATCH, mapping the total fraction of successful matches to a 1--10 scale. The second is a per-axis rubric, which scores the same five dimensions individually on a $\{0,1,2\}$ scale and sums them. Both rubrics are evaluated at both 8 and 16 uniformly-sampled frames. The denser 16-frame sampling is crucial for catching temporal failures, such as incorrect repetition counts, improper pacing, or trajectory reversals, that might be aliased or missed at 8 frames. Additionally, both rubrics incorporate an agent-equivalence clause, which dictates that the humanoid robot stands in for any human referent in the prompt, as well as strict floor caps. These floor caps ensure that a single critical failure, such as an incorrect primary action or an opposite-direction translation, cannot be masked or inflated by matching minor filler details. The final reported semantic score is the mean across all eight configuration calls (2 models $\times$ 2 rubrics $\times$ 2 frame counts). The full prompts and rubrics are detailed in \Cref{app:eval_prompts}.

\section{VLM-Judge prompts}\label{app:eval_prompts}
This appendix reproduces, verbatim, the two VLM-judge rubrics used by the ensemble of the previous section. Each rubric is issued to both models (GPT-5.5 and \texttt{gemini-2.5-flash}) at two temporal resolutions ($8$ and $16$ uniformly-sampled frames) and at temperature~$0$, giving the eight calls whose mean is the reported semantic score.

\subsection{Holistic Rubric}\label{app:p_holistic}

\begin{promptbox}
# AGENT EQUIVALENCE -- read first
The video shows a humanoid robot performing a motion. For all evaluation purposes, treat the robot as a stand-in for any human referent in the text ("a person", "the man", "they", "this person"). Robot identity is NEVER a reason to label a BODY unit as MISMATCH; only assess whether the referenced body part is correctly involved.

# YOUR ROLE
You are a strict evaluator for text-conditioned robot motion generation. Your job is to find failure modes -- do not be charitable to the generator. Judge only on observable evidence in the video. If evidence is missing or ambiguous, do NOT default to MATCH.

# INPUT
- Text: a natural language description of the target motion.
- Video: either a native mp4, or N uniformly-sampled frames in temporal order. **Treat the gaps between samples as continuous motion you did not directly observe.** A momentary airborne phase between adjacent frames is normal for walking and is NOT evidence against "walking".

# THREE-STEP PROTOCOL

## Step 1 -- Semantic Decomposition

Parse the text into atomic semantic units. **Compound verbs MUST be decomposed into component primitives.** Examples:
- "crouch-walks forward" -> ACTION 'crouch posture' + ACTION 'walking gait' + SPATIAL 'forward'
- "throws a right hook" -> ACTION 'throw' + BODY 'right arm' + ACTION 'hook (curved trajectory)'
- "steps backwards" -> ACTION 'step' + SPATIAL 'backward'
- "carries an object back and then strides sideways to the right" -> ACTION 'carry object' + SPATIAL 'back' + TEMPORAL 'then' + ACTION 'stride' + SPATIAL 'sideways right'

Categories (each unit MUST get exactly one):
- ACTION: a verb or movement primitive ("walk", "throw", "wave", "jab", "swim")
- BODY: a body part or end-effector ("right arm", "both hands"). Skip BODY units that just restate the agent -- do NOT create BODY='a person' / BODY='the man'.
- SPATIAL: a direction, trajectory, or target ("forward", "sideways right", "above head", "in front of them")
- TEMPORAL: sequencing or duration ("first ... then", "after a moment", "several times", "quickly", "then stops")
- ATTRIBUTE: stylistic quality, magnitude, or speed ("happily", "powerfully", "slightly", "joyfully"). If a word describes how a unit is executed rather than what is executed, it's an ATTRIBUTE.

Do NOT create empty filler units (no ACTION='doing', no BODY='person'). If a piece of text adds no new constraint, skip it. Aim for 2-6 units for simple prompts, up to 8 for compound prompts. Mark the FIRST ACTION unit you extract as the **primary ACTION** -- it's the one whose failure caps the score.

## Step 2 -- Per-Unit Verification

For each unit assign one label. Cite frame numbers (or timestamps if a native video) as evidence.
- MATCH: clearly satisfied with concrete visible evidence.
- PARTIAL: roughly satisfied but with a noticeable deviation. Examples: text says "quickly" but motion is normal speed; text says "small steps" but steps look normal-sized; text says "forward" and net translation is forward but minimal.
- MISMATCH: absent, contradicted, or violated. **If the unit's realization cannot be verified from the video, label MISMATCH -- absence of evidence is NOT evidence of MATCH.**

Be strict. If you find yourself hedging ("seems to roughly resemble..."), the label is at best PARTIAL.

## Step 3 -- Holistic Scoring (DETERMINISTIC)

The final score is derived from Step 2 labels -- do not pick a "vibe" number.

**a) Floor rules (apply first, take the lowest applicable cap):**
- If the video shows no motion related to the text at all -> score = 1.
- If the primary ACTION unit is MISMATCH -> score <= 3.
- If a SPATIAL unit is MISMATCH because the motion is in the opposite direction (e.g., text says "forward" but motion is backward) -> score <= 3.
- If >= 50% of units are MISMATCH -> score <= 3.
- If any unit is MISMATCH and any other unit is MISMATCH or PARTIAL -> score <= 7.
- If any unit is PARTIAL (and no MISMATCH) -> score <= 8.

**b) Candidate score from label fraction:**
Let M = number of MATCH, P = number of PARTIAL, X = number of MISMATCH, N = M+P+X.
Compute f = (M + 0.5*P) / N. Then:
- f >= 0.95 -> candidate in {9, 10}
- 0.75 <= f < 0.95 -> candidate in {7, 8}
- 0.50 <= f < 0.75 -> candidate in {5, 6}
- 0.25 <= f < 0.50 -> candidate in {3, 4}
- f < 0.25 -> candidate in {1, 2}

Inside each band, choose the higher integer if the matched units include the primary ACTION and no deviation feels severe; otherwise the lower one.

**c) Final score = min(floor_cap, candidate).**

You may freely assign 2, 4, 7, 8, 9 -- they are valid scores. Do not collapse to round numbers like 1/3/5/10.

# OUTPUT -- JSON only, no markdown fences, no commentary before or after

Output a single valid JSON object with exactly the fields shown below. Use real integers and real strings -- do NOT include angle-bracket placeholders like `<int>` in your output.

Worked example for a hypothetical prompt "walk forward slowly" where the robot walks forward but at normal speed:
{
  "semantic_units": [
    {"id": 1, "category": "ACTION", "content": "walk", "label": "MATCH", "rationale": "Robot performs a bipedal walking gait with alternating steps from frame 2 through frame 8."},
    {"id": 2, "category": "SPATIAL", "content": "forward", "label": "MATCH", "rationale": "Robot translates forward across the floor pattern; visible position change between frame 1 and frame 8."},
    {"id": 3, "category": "ATTRIBUTE", "content": "slowly", "label": "PARTIAL", "rationale": "Pace looks like normal walking, not slow."}
  ],
  "score": 7,
  "justification": "Primary ACTION 'walk' MATCH and SPATIAL MATCH, but ATTRIBUTE 'slowly' is PARTIAL. f = (2 + 0.5*1)/3 = 0.83, band {7,8}. Floor cap from one PARTIAL: <= 8. Final: 7, downgraded inside the band because the speed deviation is not minor."
}

Produce exactly that JSON shape (same keys, same nesting). Each unit id is an integer starting at 1 and incrementing. Score is an integer 1..10. No other top-level keys.
\end{promptbox}

\subsection{Per-Axis Rubric}\label{app:p_axis}

\begin{promptbox}
# AGENT EQUIVALENCE
The video shows a humanoid robot. Treat the robot as a stand-in for any "person" / "man" / "they" referenced in the text. Robot identity is never a reason to penalize.

# YOUR ROLE
You are a strict evaluator for text-conditioned robot motion. Judge only on visible evidence in the frames. Treat unsampled gaps between frames as continuous motion you didn't observe -- do NOT use those gaps as evidence of failure.

# INPUT
- Text: a description of the target motion.
- Video: a native mp4, or N uniformly-sampled frames in temporal order.

# FIVE-AXIS PROTOCOL

For each axis below, assign an integer score in {0, 1, 2}:

**Axis A -- PRIMARY ACTION**: the main verb of the text (walk, jab, wave, swim, crawl, dance, ...).
  - 2 = robot clearly performs the primary action.
  - 1 = robot does something resembling the primary action with significant deviation, OR performs an action that overlaps partially.
  - 0 = robot does something unrelated, or no discernible motion.

**Axis B -- BODY PART CORRECTNESS**: did the correct body part(s) execute the action? (e.g., "right arm" jabbed, "both hands" waved.)
  - 2 = exactly the body parts named in the text.
  - 1 = wrong-side limb, OR additional body parts also involved beyond what was specified.
  - 0 = action executed by a completely wrong body part, OR the named body part is idle while another acts.
  - If the text does not specify a body part, score 2 by default.

**Axis C -- SPATIAL CORRECTNESS**: direction, trajectory, target ("forward", "sideways right", "in front of them", "back and then to the side").
  - 2 = all spatial constraints satisfied.
  - 1 = some satisfied, some violated; OR direction approximately right but small/limited.
  - 0 = opposite or unrelated direction. If the text says "forward" and motion is backward, this is 0.
  - If the text has no spatial component, score 2 by default.

**Axis D -- TEMPORAL / COUNT / SEQUENCING**: ordering ("first ... then"), duration ("after a moment"), repetition count ("several times"), pace ("quickly").
  - 2 = repetition count and sequencing match clearly (cite frame indices showing each repetition).
  - 1 = motion type right but count / duration / order partially off (e.g., text says "several times" but motion only repeats once or twice; or text says "then stops" but motion continues).
  - 0 = no temporal/count component satisfied at all.
  - If the text has no temporal/count component, score 2 by default.

**Axis E -- ATTRIBUTE / STYLE**: speed, magnitude, intensity, mood ("slowly", "powerfully", "happily", "with force", "slightly").
  - 2 = the named quality is clearly observable.
  - 1 = partially evident (e.g., "powerfully" -- motion is fast but not visibly forceful; "joyfully" -- motion is dynamic but no clear celebratory quality).
  - 0 = the named quality is absent or contradicted.
  - If the text has no attribute, score 2 by default.

# SCORE MAPPING (DETERMINISTIC)
Let S = A + B + C + D + E (sum, range 0..10).

Final score = max(1, S). However, apply these floor caps first:
- If A = 0 -> final score <= 3.
- If C = 0 AND text has a spatial component -> final score <= 3.
- If A <= 1 -> final score <= 7.

# OUTPUT -- JSON only

Worked example for text "A person waves both hands at their side several times.":
{
  "axes": {
    "A_primary_action": {"score": 2, "rationale": "Robot raises both hands and oscillates them between frames 2-8, consistent with waving."},
    "B_body_part":      {"score": 2, "rationale": "Both hands are clearly involved (text required 'both hands')."},
    "C_spatial":        {"score": 1, "rationale": "Hands move at chest/face level rather than 'at their side'; minor spatial deviation."},
    "D_temporal":       {"score": 1, "rationale": "Only ~2 visible waves in the clip; text says 'several times' which implies more than 2."},
    "E_attribute":      {"score": 2, "rationale": "No attribute in text -- default 2."}
  },
  "sum": 8,
  "floor_cap_applied": "none",
  "score": 8,
  "justification": "Action and body parts correct; minor spatial and count deviations bring this to 8."
}

Produce a JSON object with exactly the same shape. Use real integers 0/1/2 for axis scores. Score is integer 1..10. No markdown fences, no commentary.
\end{promptbox}

\end{document}